\documentclass[letterpaper]{article} 
\usepackage{aaai23}  
\usepackage{times}  
\usepackage{helvet}  
\usepackage{courier}  
\usepackage[hyphens]{url}  
\usepackage{graphicx} 
\usepackage{natbib}  
\usepackage{caption} 
\usepackage{amsmath}
\usepackage{amssymb}
\usepackage{mathtools}
\usepackage{amsthm}
\usepackage{algorithm}
\usepackage{algorithmic}
\usepackage{bm}
\usepackage{subfigure}

\usepackage{booktabs}
\usepackage{multirow}

\newtheorem{theorem}{Theorem}[section]
\newtheorem{proposition}[section]{Proposition}
\newtheorem{lemma}[theorem]{Lemma}

\newtheorem{definition}[section]{Definition}

\newcommand{\norm}[1]{\left\lVert#1\right\rVert}

\newcommand{\bignorm}[1]{\big\lVert#1\big\rVert}

\newcommand{\bmf}[1]{\boldsymbol#1}

\def\R{\mathbb{R}}
\usepackage{newfloat}
\usepackage{listings}
\DeclareCaptionStyle{ruled}{labelfont=normalfont,labelsep=colon,strut=off} 
\floatstyle{ruled}
\newfloat{listing}{tb}{lst}{}
\floatname{listing}{Listing}
\title{Bayesian Federated Neural Matching that Completes Full Information}
\author{
    Peng Xiao,\textsuperscript{\rm 1}
    Samuel Cheng\textsuperscript{\rm 2}\thanks{ is the corresponding author.}
}

\affiliations{
    \textsuperscript{\rm 1}Department of Computer Science and Technology, Tongji University\\

    Shanghai, China\\
    \textsuperscript{\rm 2}School of Electrical and Computer Engineering, University of Oklahoma \\
    Oklahoma City, US \\
    phd.xiaopeng@gmail.com\\
}

\usepackage{bibentry}

\begin{document}

\maketitle

\begin{abstract}
Federated learning is a contemporary machine learning paradigm where locally trained models are distilled into a global model. Due to the intrinsic permutation invariance of neural networks, Probabilistic Federated Neural Matching (PFNM) employs a Bayesian nonparametric framework in the generation process of local neurons, and then creates a linear sum assignment formulation in each alternative optimization iteration. But according to our theoretical analysis, the optimization iteration in PFNM omits global information from existing. In this study, we propose a novel approach that overcomes this flaw by introducing a Kullback-Leibler  divergence penalty at each iteration.
The effectiveness of our approach is demonstrated by experiments on both image classification and semantic segmentation tasks.
\end{abstract}

\section{Introduction}
\label{sec:introduction}
\par The recent decade has seen a rapid advancement in artificial intelligence, particularly in the subfield of deep learning. This is large because there is an abundance of data available. However, the increasing privacy concerns present a barrier to some datasets' accessibility in a number of applications, particularly in the medical and financial areas. Therefore, federated learning (FL), which involves learning a model from disparate sets of data, is suggested to address this issue.

\par Federated learning is a learning paradigm where locally learned models are combined into a shared global model. Neural networks are naturally used in federated learning since they play a vital role in deep learning. Considering the inherent permutation invariance of a neural network, it is reasonable to match neurons of distinct local models before aggregating them. To do this, Probabilistic Federated Neural Matching (PFNM) \cite{yurochkin2019bayesian} builds a Bayesian nonparametric framework to match and merge these neurons. In the process of modeling, PFNM formally characterizes the generative process of local neurons through the Beta-Bernoulli process (BBP) \cite{2007Hierarchical}, and treats the local neurons as noisy realizations of latent global neurons. In the process of optimizing, PFNM iteratively maximizes the posterior estimation of latent global neurons, through solving a linear sum assignment formulation of global neurons and current local neurons. PFNM subsequently extends in modern architectures such as convolutional neural networks (CNNs) and long short-term memory (LSTMs) \cite{wang2020federated}, and its variants are also utilized in aggregating various statistical models such as Gaussian topic models, hierarchical Dirichlet process based
hidden Markov models \cite{yurochkin2019statistical,yurochkin2018scalable} etc.

\par In this paper, we theoretically prove the drawback of the optimizing process in PFNM under its probabilistic framework and address it by introducing a Kullback-Leibler (KL) divergence penalty. Our theory can also be generalized in all PFNM variants. Specifically, the contributions in this work include: 1) We theoretically prove that the linear sum formulation in the optimizing process omits the information of global neurons under the Bayesian framework. 2) We fix the missing global information by introducing a KL divergence penalty to complete the full information. 3) In experiment, we not only demonstrate the effectiveness of our approach on three image datasets, but also extend the matching type federated learning algorithm to neural network with batch normalization layer.

\section{Related Works}
\subsubsection{Federated Learning}
The initial aggregation method in FL is FedAvg \cite{mcmahan2017communication} in which parameters of local models are averaged coordinate-wisely. But the performance of FedAvg deteriorates significantly in non-i.i.d (Independent and Identically Distributed) data \cite{karimireddy2020scaffold,deng2021distributionally,xiao2020averaging}. Subsequent improved methods starts from different perspectives. FedProx \cite{li2018federated} adds a proximal term in local training cost to keep dissimilarity between local models in lower bound. SCAFFOLD \cite{karimireddy2020scaffold} uses control variates to correct the drift in local updates. FedPD \cite{zhang2020fedpd} proposes a primal-dual optimization strategy to alleviate deterioration in non-i.i.d data. Agnostic federated learning \cite{mohri2019agnostic} optimizes a centralized distribution by maximize the worst local model. Federated multi-task learning \cite{smith2017federated} applies a multi-task learning mechanism to customize local models. Several studies further extend \textit{knowledge distillation} \cite{hinton2015distilling,bucilua2006model,schmidhuber1992learning} in federated learning. And the key idea is to employ the knowledge of pre-trained teacher neural networks (local models) to learn a student neural network (global model) \cite{lin2020ensemble, chen2020fedbe, zhu2021data}.  However, all those methods above don't considering the permutation invariance in neural network.

\subsubsection{Parameter Matching} There are also other researches about match the parameters. In \cite{singh2020model}, the authors align neurons across different NNs by minimizing an optimal transportation cost matrix. However, it fixes the number of global neurons, making it impractical when data from different local models is extremely heterogeneous.
Some work \cite{claici2020model} 
uses variational inference to 
optimize the assignments between global and local components under a KL divergence. However, the optimization process is complex, and calculating variational inference is difficult.

\section{Problem Formulation}
Given $S$ fully connected (FC) Neural Networks (NNs) with a hidden layer trained through different datasets: $f_s(\bmf{x}) = \bmf{W}_s^{(1)}\sigma(\bmf{W}_s^{(0)}\bmf{x}), \text{for } s = 1, \cdots, S$ (for simplifying notation, biases are omitted to ), where $\sigma(\cdot)$ is the nonlinear activation function, $\bmf{W}_s^{(1)} \in \R^{K \times J_s}$ and $\bmf{W}_s^{(0)} \in \R^{J_s \times D}$ are the weights; with $D$ being the 
input dimension, $K$ being the output dimension (i.e., number of classes), and $J_s$ being the number of neurons on the hidden layer of $s$-th NN. Neuron indexed by $j$ is viewed as a concatenated vector $\bmf{w}_{sj}^{T} = [\bmf{W}_{s, j \cdot}^{(0)}, {\bmf{W}_{s, \cdot j}^{(1)}}^T]$,  where $j \cdot$ and $\cdot j$ denote the $j$th row and column correspondingly. And $s$-th FCNN can also be viewed as a collection of neurons  $\{\bmf{w}_{sj} \in \R^{D + K}\}_{j=1}^{J_s}$. 
In federated learning, we want to learn a global neural network with weights $\Theta^{(0)} \in \R^{J \times D}, \Theta^{(1)} \in \R^{K \times J}$, where $J \ll \sum_{s=1}^S J_s$ is an inferred variable denoting the number of global neurons.
\subsubsection{Permutation Invariance} Expanding the preceding expression of a FCNN: $f_s(\bmf{x}) = \sum_{j=1}^{J_s} \sigma(\langle\bmf{W}_{s, j\cdot}^{(0)},\bmf{x}^T\rangle)\bmf{W}_{s, j \cdot}^{(1)}$. Summation is a permutation invariant operation, thus any permutation $\tau(1,\cdots,Js)$ of rows of $\bmf{W}_s^{(0)}$ and columns of $\bmf{W}_s^{(1)}$, i.e. the neurons, will not affect the output for any input $\bmf{x}$. Due to the permutation invariance, a neuron indexed by $j$ from one FCNN is unlikely to correspond to a neuron with the same index from another FCNN. Thus, we should match neurons from different FCNNs before aggregate them into a collection of global neurons $\{\bmf{\theta}_{i} \in \R^{D + K}\}_{i=1}^{J}$.
\section{Probabilistic Modeling}
PFNM models the generative process of observed local neurons using a Beta Bernoulli process, which is described in the Appendix A, to infer the global model while accounting for the inherent permutation invariance.  First, consider the collection of global neurons as prior which are sampled from a Beta process with a base measure: $M = \sum_i m_i \delta_{\bmf{\theta}_i} \sim \text{BP}(1,\gamma_0 H)$ and $\bmf{\theta}_i \sim H$, 
where $\gamma_0$ is the mass parameter, $m_i$ are the stick-breaking weights, $H$ is the base measure and chosen as a multivariate Gaussian distribution $H = \mathcal{ N}(\bmf{\mu}_0, \bm{\Sigma_0} )$ with $\bmf{\mu}_0 \in \R^{D+K}$ and diagonal $\bm{\Sigma_0}=\bmf{I}\sigma_0^2$. 
\par The Bernoulli process is then used by each local model to select a subset of global neurons:$
\mathcal{T}_s = \sum_{i} a_{si} \delta_{\bmf{\theta}_i} \vert M \sim \text{BeP}(M) \text{ for } s=1,\ldots,S$, where $a_{si} \vert m_i \sim \operatorname{Bern}(m_i) $ is a random measure representing a subset of global neurons contained in local model $s$. Considering the inherent permutation invariance of local model, it denotes $\bmf{A} = \{A_{ij}^s\}_{s,i,j}$ as the assignment variables, $\sum_{j} A^s_{ij} = a_{si} \in \{0,1\}$, $\sum_{i} A^s_{ij} = 1$. Finally, observed local neurons in model $s$ are treated as noisy measurements of global neurons under permutation invariance:
\begin{equation}
\label{eq:assignment_parametrize}
    \begin{split}
        \bmf{w}_{sj}\mid  \bmf{A}^s, \bmf{\theta} \sim \mathcal{ N}(\sum_i A^s_{ij}\bmf{\theta}_i, \bm{\Sigma_s}) \text{ for } j=1,\ldots,J_s, 
    \end{split}
\end{equation}
where $J_s = card(\mathcal{T}_s)$, $\bm{\Sigma_s} = \bm{I}\sigma_s^2$ is also diagonal and represents the noise. The noise is usually caused by estimation error due to finite sample sizes or variations in the distribution of each local dataset. $A_{ij}^s = 1$ indicates that $\bmf{w}_{sj}$ is matched to $\bmf{\theta}_i$, i.e. $\bmf{w}_{sj}$ is the local neuron realization of the global neuron $\bmf{\theta}_i$; $A_{ij}^s = 0$ indicates the inverse. 
\par After the modeling, it is natural to infer the global neurons by maximizing the posterior probability of the global neurons given local neurons (likelihoods) under the permutation invariance:
\begin{equation}\label{max_pos}
\begin{split}
\max_{\{ \bmf{\theta}_i \}, \{ \bmf{A}^s \} }  & P(\{ \bmf{\theta}_i\}, \{\bmf{A}^s\} | \{\bmf{w}_{sj}\}) \propto \\
 & P(\{\bmf{w}_{sj}\} | \{\bmf{\theta}_i\}, \{\bmf{A}^s\} )P(\{\bmf{A}^s\})P(\{\bmf{\theta}_i\}),
\end{split}
\end{equation}
define $\bmf{Z}_i = \{(s, j) \vert A_{ij}^s = 1 \}$ be the index set of local neurons assigned to $i$-th global neuron, and taking negative logarithm it can obtain:
\begin{equation}\label{neg_log_pos}
\begin{split}
\min_{\{ \bmf{\theta}_i \}, \{ \bmf{A}^s \} }  & - \log(P(\{ \bmf{A}^s \})) \\
& -\sum_i \bigg ( \sum_{z \in \bmf{Z}_i} \log(p(\bmf{w}_{z} \vert \bmf{\theta}_i)) + \log(p(\bmf{\theta}_i)) \bigg )
\end{split}
\end{equation}
where $P(\{ \bm{A}^s \})$ is interpreted by Indian Buffet Process (IBP) and demonstrated in Appendix B.
Given $\{ \bm{A}^s \}_{s=1}^S$, the closed form of  $\{ \bmf{\theta}_i \}$ can be estimated through the Gaussian-Gaussian conjugacy:
\begin{equation}\label{Gauss_conj}
\bmf{\theta}_i =  \frac{\bmf{\mu}_0 / \sigma_0^2 + \sum_{s,j}\bmf{A}_{i,j}^s \bmf{w}_{sj} / \sigma^2_s}{1 / \sigma_0^2 + \sum_{s,j} \bmf{A}^s_{i,j} / \sigma^2_s} \, \text{ for } i = 1, \cdots, J.
\end{equation}
Taking equation~(\ref{Gauss_conj}) into objective~(\ref{neg_log_pos}), it can cast optimization only with respect to $\{ \bm{A}^s \}_{s=1}^S$. Unfortunately, solving all local assignments together leading an NP-hard combinatorial optimization problem. Thus, PFNM applies an alternative optimization process to solve the problem.

\subsubsection{Alternative Optimization}
\par PFNM iteratively optimizes one local assignment variable $\bmf{A}^{s'}$ at a time by fixing all other assignment variables $\{A^{s}_{i,j} \}_{i,j, s \in -s'}$, where $-s'$ denotes "all but $s'$". It aims to formulate a linear sum assignment problem with current local assignment $\sum_{i,j}C^{s'}_{i,j}A^{s'}_{i,j}$ in each iteration, where $\{C^{s'}_{i,j} \}_{i,j}$ denotes the cost specification. The Hungarian algorithm is then used to solve the linear sum assignment problem. In each iteration, it divides terms in objective~(\ref{neg_log_pos}) into two parts: one is $i = 1, \cdots, J_{-s'}$, where $J_{-s'} = \max\{i: \bmf{A}^{s}_{ij} = 1$, for $s\in -s', j = 1, \dots, J_{s} \}$, this part denotes active global neurons estimated from $-s'$; another is $i = J_{-s'}+1, \cdots, J_{-s'}+J_{s'}$, it denotes new global neuron from current local model. Then it has the following proposition describes the assignment cost $\{C^{s'}_{i,j}\}_{i,j}$:
\begin{proposition}
\label{pro:pfnm_cost}
The assignment cost specification $C^{s'}_{i,j}$ for finding $\{ \bmf{A}^{s'} \}$ is $C_{i,j}^{s'}=$
\begingroup
\allowdisplaybreaks
\begin{equation}
\label{pfnm_cost}
\begin{cases}
\begin{split}
2\log{\frac{S-n_i^{-s'}}{n_i^{-s'}}}   & -\frac{\bignorm{\frac{\bmf{\mu}_0}{\sigma_0^2} + \frac{\bmf{w}_{s'j}}{\sigma^2_{s'}} + \sum_{s\in {-s'},j} \bmf{A}_{i,j}^s\frac{\bmf{w}_{sj}}{\sigma_s^2}}^2}{\frac{1}{\sigma_0^2} + \frac{1}{\sigma^2_{s'}} + \sum_{s\in {-s'},j} \bmf{A}_{i,j}^s\frac{1}{\sigma_s^2}} \\
& + \frac{\bignorm{\frac{\bmf{\mu}_0}{\sigma_0^2} + \sum_{s\in {-s'},j} \bmf{A}_{i,j}^s\frac{\bmf{w}_{sj}}{\sigma_s^2}}^2}{\frac{1}{\sigma_0^2} + \sum_{s\in {-s'},j} \bmf{A}_{i,j}^s\frac{1}{\sigma_s^2}} , i \leq J_{-s'} 
\end{split}
\\
\begin{split}
2\log{\frac{i-J_{-s'}}{\gamma_0/S}} & -\frac{\bignorm{\frac{\bmf{\mu}_0}{\sigma_0^2} + \frac{\bmf{w}_{s'j}}{\sigma^2_{s'}} }^2}{\frac{1}{\sigma_0^2} + \frac{1}{\sigma^2_{s'}} } \\
& + \frac{\bignorm{\frac{\bmf{\mu}_0}{\sigma_0^2} }^2}{\frac{1}{\sigma_0^2} }, \qquad\quad J_{-s'}<i \leq J_{-s'} + J_{s'},
\end{split}
\end{cases}
\end{equation}
\endgroup
where $n_i^{-s'} = \sum_{s \in -s',j} \bmf{A}^s_{i,j}$ denotes the number of local neurons were assigned to global neuron $i$ outside of $s'$. 
\end{proposition}

The proof can be found in Appendix B and \cite{yurochkin2019bayesian}.
\section{Analysis of PFNM}
Before we discuss, we have the following proposition:
\begin{proposition}\label{pro:gauss_conjugate}
For any prior of a global neuron $\bmf{\theta}_i \sim \mathcal{ N}(\bmf{\mu}_0, \bm{I}\sigma_0^2 )$, when it is assigned with a likelihood observed local neuron $\bmf{w}_{s'j} \vert \bmf{\theta}_i \sim \mathcal{ N}(\bmf{\theta}_i, \bm{I}\sigma_{s'}^2)$, the posterior distribution of this global neruon is $\bmf{\theta}_i \vert \bmf{w}_{s'j} \sim \mathcal{ N}\Bigl(\frac{\bmf{\mu}_0/\sigma_0^2 + \bmf{w}_{s'j}/\sigma^2_{s'}}{1/\sigma_0^2 + 1/\sigma^2_{s'}}, \bm{I}\bigl(\frac{1}{1/\sigma_0^2 + 1/\sigma^2_{s'}}\bigr)\Bigr)$; when it has been assigned local neurons from other assignments $\bmf{Z}_i^{-s'} = \{(s, j) \vert A_{ij}^s = 1, s \in -s'\}$, the posterior distribution is $\bmf{\theta}_i \vert \bmf{Z}_i^{-s'} \sim \mathcal{ N}(\frac{\bmf{\mu}_0/\sigma_0^2 + \sum_{s\in {-s'},j} A_{i,j}^s\bmf{w}_{sj}/\sigma_s^2}{1/\sigma_0^2 + \sum_{s\in {-s'},j} A_{i,j}^s/\sigma_s^2}, \bm{I}(\frac{1}{1/\sigma_0^2 + \sum_{s\in {-s'},j} A_{i,j}^s/\sigma_s^2}))$.
\end{proposition}
\par The proof can be found in Appendix C.

\par Terms of equation~(\ref{pfnm_cost}) on the left are due to $P(\bmf{A}^{s'} \vert \{ \bmf{A}^{-s'} \})$. At initial stage of iteration all $n_i^{-s'}$ are identical. So the main differences compared by cost $C^{s'}_{i,j}$ induced from right term. 
\begin{definition}\label{defi:sms}
(Standardized mean square) For any multivariate Gaussian distribution $\mathcal{N}(\bmf{\mu}, \bm{I}\sigma^2)$, we call $\norm{\bmf{\mu}}^2/\sigma^2$ as standardized mean square of this distribution, and denote it as $\mathbb{S}(\mathcal{N}) = \norm{\bmf{\mu}}^2/\sigma^2$.
\end{definition}
The proposition~\ref{pro:gauss_conjugate} shows that when $i \leq J_{-s'}$, the right term is the difference in standardized mean square between distribution $p(\bmf{\theta}_i \vert \bmf{Z}_i^{-s'})$ and this distribution assigned with local likelihood $p(\bmf{\theta}_i \vert \bmf{Z}_i^{-s'}, \bmf{w}_{s'j})$; when $J_{-s'}<i \leq J_{-s'} + J_{s'}$, the right term is the difference of standardized mean square between prior distribution $p(\bmf{\theta}_i)$ and this distribution assigned with local likelihood $p(\bmf{\theta}_i \vert \bmf{w}_{s'j})$. So, by comparing each cost term $C^{s'}_{i,j}$, it essentially compares the standardized mean square difference between the $i$-th global neuron distribution (either estimated from $-s'$ or a new one from the current local model) and this distribution assigned with local neuron $w_{s'j}$. However, we have the following proposition shows the issue caused by the cost term:
\begin{proposition}\label{SMS_differnce_issue}
For a global neuron distribution $p(\bmf{\theta}) = \mathcal{ N}(\bmf{\mu}_0, \bm{I}\sigma_0^2)$, after it is assigned with a likelihood $p(\bmf{w} \vert \bmf{\theta}) = \mathcal{ N}(\bmf{\theta}, \bm{I}\sigma^2)$, the standardized mean square difference between $p(\bmf{\theta})$ and $p(\bmf{\theta} \vert \bmf{w})$ is proportion to $\norm{\bmf{w}}^2/\sigma^2$.
\end{proposition}
\begin{proof}
Denote the mean and covariance matrix of $p(\bmf{\theta} \vert \bmf{w})$ as $\tilde{\bmf{\mu}}$ and $\bmf{I}\tilde{\sigma}^2$. From definition~\ref{defi:sms}, we can write the the standardized mean square difference between $p(\bmf{\theta})$ and $p(\bmf{\theta} \vert \bmf{w})$ as
\begin{align}\label{eq:sms_differ}
 &\mathbb{S}(p(\bmf{\theta})) - \mathbb{S}(p(\bmf{\theta} \vert \bmf{w})) \nonumber\\
  = & \frac{\norm{\bmf{\mu}_0}^2}{\sigma_0^2} - \frac{\norm{\tilde{\bmf{\mu}}}^2}{\tilde{\sigma}^2} \nonumber\\ 
 = & \log \frac{\exp{(\frac{-\bmf{\theta}^T\bmf{\theta}+2\bmf{\theta}^T\bmf{\mu}_0}{2\sigma_0^2})}}{(2\pi\sigma_0^2)^{\frac{D+K}{2}}p(\bmf{\theta})} - \log \frac{\exp{(\frac{-\bmf{\theta}^T\bmf{\theta}+2\bmf{\theta}^T\tilde{\bmf{\mu}}}{2\tilde{\sigma}^2})}}{(2\pi\tilde{\sigma}^2)^{\frac{D+K}{2}}p(\bmf{\theta} \vert \bmf{w})} \nonumber\\
  = & \log \frac{\exp{(\frac{-\bmf{\theta}^T\bmf{\theta}+2\bmf{\theta}^T\bmf{\mu}_0}{2\sigma_0^2})}}{(2\pi\sigma_0^2)^{\frac{D+K}{2}}p(\bmf{\theta})} \cdot \frac{(2\pi\tilde{\sigma}^2)^{\frac{D+K}{2}}p(\bmf{\theta} \vert \bmf{w})}{\exp{(\frac{-\bmf{\theta}^T\bmf{\theta}+2\bmf{\theta}^T\tilde{\bmf{\mu}}}{2\tilde{\sigma}^2})}}.
\end{align}
Because
\begin{align}
    & (\frac{\tilde{\sigma}^2}{\sigma_0^2})^{\frac{D+K}{2}} \cdot \exp(\frac{-\bmf{\theta}^T\bmf{\theta}+2\bmf{\theta}^T\bmf{\mu}_0}{2\sigma_0^2}- \frac{-\bmf{\theta}^T\bmf{\theta}+2\bmf{\theta}^T\tilde{\bmf{\mu}}}{2\tilde{\sigma}^2}) \nonumber\\
    = & (\frac{\tilde{\sigma}^2}{\sigma_0^2})^{\frac{D+K}{2}} \cdot \exp(\bmf{\theta}^T\bmf{\theta}(-\frac{1}{2\sigma_0^2} + \frac{1}{2\tilde{\sigma}^2}) + \bmf{\theta}^T(\frac{\bmf{\mu}_0}{\sigma_0^2} - \frac{\tilde{\bmf{\mu}}}{\tilde{\sigma}^2})),
\end{align}
from proposition~\ref{pro:gauss_conjugate}, we can obtain
\begin{align}
    \frac{1}{2}(-\frac{1}{\sigma_0^2} + \frac{1}{\tilde{\sigma}^2}) = \frac{1}{2 \sigma^2}, \qquad \frac{\bmf{\mu}_0}{\sigma_0^2} - \frac{\tilde{\bmf{\mu}}}{\tilde{\sigma}^2} = -\frac{\bmf{w}}{\sigma^2}, \\
    (\frac{\tilde{\sigma}^2}{\sigma_0^2})^{\frac{D+K}{2}} = (1+\frac{\sigma_0^2}{\sigma^2}))^{-\frac{D+K}{2}} \text{ is a constant in PFNM setting,} \nonumber
\end{align}
and from Bayesian theory
\begin{equation}\label{eq:baye_theory}
    p(\bmf{\theta} \vert \bmf{w}) \propto p(\bmf{\theta}) p(\bmf{w} \vert \bmf{\theta}),
\end{equation}
thus we have
\begin{align}
    \mathbb{S}(p(\bmf{\theta})) - \mathbb{S}(p(\bmf{\theta} \vert \bmf{w})) & \propto \log\frac{p(\bmf{w})}{\exp(\frac{-\bmf{\theta}^T\bmf{\theta}+2\bmf{\theta}^T\bmf{w}}{2\sigma^2})} \nonumber\\
    & \propto -\frac{\bmf{w}^T\bmf{w}}{2\sigma^2}.
\end{align}
\end{proof}
From proposition~\ref{SMS_differnce_issue}, We can know that the differences compared by each $C^{s’}_{ij}$ are only related to local neuron. As a consequence, for one fixed local neuron $\bmf{w}_{s'j}$ and two different global neurons $\bmf{\theta}_i$ and $\bmf{\theta}_{i^*}$, the cost does not discriminate global neurons $\bmf{\theta}_i$ and $\bmf{\theta}_{i^*}$. Hence, the cost specifications induced from original PFNM can't find optimal solutions for $\{ \bmf{A}^{s'} \}$.
\subsection{Kullback-Leibler Divergence}
From the above analysis, we know that the cost specification $C^{s'}_{i,j}$ induced from PFNM omits information of global neuron, i.e. the $i$ index in the cost. So, it is natural to fix the issue by adding a regularized term that contains both information of $i$ and $j$. As we point out before, the cost term $C^{s'}_{i,j}$ essentially is the standardized mean square difference between $i$-th global distribution and this distribution assigned with local neuron $w_{s'j}$. From the probabilistic perspective, minimizing the cost is meant to hope these two distributions are as close as possible. So how to measure the distance between two distributions? Kullback-Leibler becomes the first answer that gets into mind. Can KL-divergence fix the issues caused by PFNM? To answer this question, we have the following proposition:
\begin{proposition}\label{propo:KL_term}
For a global neuron distribution $p(\bmf{\theta})$, after it is assigned with a likelihood $p(\bmf{w} \vert \bmf{\theta})$, the Kullback–Leibler Divergence Penalty between $p(\bmf{\theta})$ and $p(\bmf{\theta} \vert \bmf{w})$ is equal up to $\mathbb{E}_{p(\bmf{\theta})} \bigl[ \log p(\bmf{w}\mid \bmf{\theta}) \bigr].$
\end{proposition}
\begin{proof}
 \begin{align}
     &  \operatorname{KL}\bigl(p(\bmf{\theta}) \Vert p(\bmf{\theta} \vert \bmf{w}) \bigr) \nonumber \\
    = & \int p(\bmf{\theta}) \log \frac{p(\bmf{\theta})}{p(\bmf{\theta} \vert \bmf{w})} \,\rm{d}\bmf{\theta}  \nonumber\\
 \end{align}
 and from Bayesian theory
\begin{equation}\label{eq:baye_theory}
    p(\bmf{\theta} \vert \bmf{w}) \propto p(\bmf{\theta}) p(\bmf{w} \vert \bmf{\theta}),
\end{equation}
so we have
\begin{align}
    &  \operatorname{KL}\bigl(p(\bmf{\theta}) \Vert p(\bmf{\theta} \vert \bmf{w}) \bigr)
    \cong  -\mathbb{E}_{p(\bmf{\theta})} \bigl[ \log p(\bmf{w}\mid \bmf{\theta}) \bigr]
\end{align}
\end{proof}
From above proposition, KL divergence between two distribution is essentially equal up to the expectation of the assigned local likelihood distribution, and the expectation is taken over the global distribution $p(\bmf{\theta})$. That means the KL divergence contains both information of local and global neuron.
\par Thus, to fix the drawback induced from PFNM,  we can formulate a new cost specifications that regularize original cost specification with the KL penalty, i.e.,  $\widetilde{\bmf{C}}^{s'}_{i,j} = $
\begin{equation}
\label{eq:kl_cost}
\begin{cases} 
\begin{split}
\bmf{C}^{s'}_{i,j} + \lambda
\operatorname{KL}\Bigl(p(\bmf{\theta}_i \vert \bmf{Z}_i^{-s'}) \big\Vert p(\bmf{\theta}_i \vert \bmf{Z}_i^{-s'}, \bmf{w}_{s'j})\Bigr),
i \leq J_{-s'}, \quad
 \end{split}\\
\begin{split}
\bmf{C}^{s'}_{i,j} + \lambda
\operatorname{KL}\Bigl(p(\bmf{\theta}_i) \big\Vert p(\bmf{\theta}_i \vert \bmf{w}_{s'j})\Bigr), & J_{-s'} < i \leq J_{-s'} + J_{s'}.
\end{split}
\end{cases}    
\end{equation}
\par where the coefficient $\lambda$ is the adjusting ratio. And the KL divergence between two multivariate normal distributions can be calculated by the following Lemma:
\begin{lemma}
The Kullback-Leibler divergence between $\mathcal{ N}_x(\bmf{\mu}_x, \bmf{\Sigma}_x)$ and $\mathcal{ N}_y(\bmf{\mu}_y, \bmf{\Sigma}_y)$, is:
\begin{equation}
    \frac{1}{2}[\operatorname{tr}(\bmf{\Sigma}_y^{-1}\bmf{\Sigma}_x) + (\bmf{\mu}_y - \bmf{\mu}_x)^T\bmf{\Sigma}_y^{-1}(\bmf{\mu}_y - \bmf{\mu}_x)-(D+K) + \ln{\frac{|\bmf{\Sigma}_y|}{|\bmf{\Sigma}_x|}}],
\end{equation}
where $D+K$ is the dimension of $\bmf{\mu}_x$ and $\bmf{\mu}_y$.
\end{lemma}
The proof can be found in \cite{duchi2007derivations}.
\par We call our new method as neural aggregation with full information (NAFI). And the process of NAFI can be summarized in algorithm~\ref{algori:nafi}. 
\par As demonstrated by the study in \cite{wang2020federated}, directly applying the matching algorithms fails on deep architectures designed for more complex tasks. To address this issue, we extend NAFI to a layer-wise matching scheme, which can be found in Appendix D.
\begin{algorithm}
\caption{The process of NAFI}
\label{algori:nafi}
\begin{algorithmic}[1]
    \REQUIRE ~~\\
    Local weights $\bmf{w}_{sj}$ from $S$ local models;
    \ENSURE ~~\\
    Matching assignments $\{ \bmf{A}^s \}_{s=1}^S$, global weights $\{ \bmf{\theta}_i$ \};
    \FOR{$t$ = 1, 2, 3, $\cdots$}
        \FOR{$s'$ = 1, 2, $\cdots$, $S$}
            \STATE  Giving assignments $\{ \bmf{A}^{s}_{i,j} \}_{i,j, s \in -s'}$, acquiring corresponding global model via equations~(\ref{Gauss_conj}):  $\tilde{\bmf{\theta}}_i = \frac{\bmf{\mu}_0/\sigma_0^2 + \sum_{s\in {-s'},j} \bmf{A}_{i,j}^s \bmf{w}_{sj}/\sigma_s^2}{1/\sigma_0^2 + \sum_{s\in {-s'},j} \bmf{A}_{i,j}^s/\sigma_s^2}$;
 	        \STATE Acquire the assignment cost $\{\widetilde{\bmf{C}}^{s'}_{i,j}\}_{i,j}$ via equation~(\ref{eq:kl_cost});
 	        \STATE Obtain $\{ \bmf{A}^{s'}_{i,j}  \}_{i,j}$ by solving the linear assignment problem via Hungarian algorithm.
        \ENDFOR
        \STATE Apply $\{ \bmf{A}^s \}_{s=1}^S$ to update the global model.
    \ENDFOR
\end{algorithmic}
\end{algorithm}

\section{Experiments}
To evaluate the efficiency of the proposed NAFI, we presents an empirical study of NAFI and compares it with PFNM, FedAvg \citep{mcmahan2017communication}, FedProx \citep{li2018federated} and optimal transport fusion \cite{singh2020model} (which we refer as OT fusion in remaining part). The experiments are carried out over three datasets with three types of neural networks: FCNN, shallow CNN and U-net. The experiments below indicate that our algorithm can aggregate multiple neural networks into a more efficient global neural network\footnote{https://github.com/XiaoPeng24/NAFI}. 

\subsubsection{Datasets, models and metrics}  Our algorithm is evaluated on three datasets: MINST, CIFAR 10 and Carvana Image Masking Challenge (CIMC). MINST and CIFAR-10 are image classification datasets and each contains ten classes on handwriting digits and objects in real life respectively. CIMC is a binary semantic segmentation dataset consisting of photos of cars, and the task is to split out the car and the background.  For MNIST, we use an FCNN model and evaluate it with accuracy; for CIFAR 10, we use a ConvNet with 3 convolutional and 2 fully-connected layers and evaluate it with accuracy; and for CIMC, we use the U-net \citep{ronneberger2015u} and evaluate it with dice coefficient. 

\subsubsection{Partition strategies of local data} To simulate a federated learning scenario,
we partition each dataset in a heterogeneous strategy in which the number of data points and class proportions in each local model is unbalanced. We follow prior works \cite{yurochkin2018scalable} in the heterogeneous partition of local data for three datasets, which apply $K$-dimensional Dirichlet distribution $Dir(\alpha)$ to create non-identical independent distribution (iid) data, with a smaller $\alpha$ indicating higher data heterogeneity. Specifically, for dataset with class number $K$, we sample the proportion of the instances of class $k$ to local model $s$, $p_{k,s}$, via $p_{k,s} \sim Dir_k(\alpha)$. For MNIST and CIFAR10, $K = 10$. For CIMC, $K = 1$ as it is a binary semantic segmentation dataset. In each dataset, we execute 5 trials to obtain the mean and standard deviation of the performances.

\subsubsection{Baselines}  Our method is compared to the original PFNM, FedAvg, OT fusion and FedProx. FedAvg and FedProx are executed in local neural networks that have been trained using the same random initialization as proposed by  \cite{mcmahan2017communication}. We note that while a federated averaging variant without shared initialization is likely to be more realistic when attempting to aggregate pre-trained models, it performs significantly worse than all other baselines. In CNN architectures, the original PFNM we compared is actually its deep model extensive version—FedMA \cite{wang2020federated}. Considering the high heterogeneity in each trial, we set $\lambda$ in a wide grid (from $10^{-8}$ to 1) and choose the best result for NAFI. We subsequently check the sensitivity of $\lambda$ and find it only have tiny fluctuation.

\begin{table}[htb]
\caption{Hyperparameter settings for training neural networks}
 \label{table:train_setting}
\begin{center}
 \begin{tabular}{cccc}
  \toprule
  & MNIST & CIFAR-10 & CIMC\\
  \cmidrule(r){1-4}
  Model & FCNN & ConvNet & U-net \\
  Optimizer & Adam & SGD & RMSprop \\
  Learning rate & 0.01 & 0.01 & 0.01\\ 
  Size of minibatch & 32 & 32 & 3 \\ 
 Epochs & 10 & 10 & 3 \\ 

 \bottomrule
 \end{tabular}
\end{center}
\end{table}

\begin{table*}[h]
\caption{Test accuracy of FCNN in MNIST}
 \label{table:mnist}
\begin{center}
 \begin{tabular}{ccccccccc}
  \toprule
   & S & N & Local NN & FedAvg & FedProx & OT fusion & PFNM & NAFI\\
  \midrule
   \multirow{4}{*}{nets} & 15 & 1 & 71.9 $\pm$ 2.57 & 75.47 $\pm$ 5.90 & 75.65 $\pm$ 5.93 & 81.44 $\pm$ 3.21 & 83.93 $\pm$ 0.14 & \textbf{87.34 $\pm$ 0.57}\\
   ~ & 20 & 1 & 69.44 $\pm$ 2.50 & 75.02 $\pm$ 4.03 & 75.17 $\pm$ 3.99 & 82.71 $\pm$ 3.57 & 83.23 $\pm$ 3.31 & \textbf{86.73 $\pm$ 2.15}\\
   ~ & 25 & 1 & 67.99 $\pm$ 2.00 & 75.46 $\pm$ 3.33 & 75.24 $\pm$ 3.30 & 82.53 $\pm$ 3.05 & 84.87 $\pm$ 1.66 & \textbf{86.46 $\pm$ 2.56} \\
   ~ & 30 & 1 & 65.90 $\pm$ 2.66 & 73.43 $\pm$ 4.49 & 73.11 $\pm$ 4.44 & 82.74 $\pm$ 2.41 & 83.39 $\pm$ 2.23 & \textbf{86.23 $\pm$ 2.57} \\
   \midrule
   \multirow{2}{*}{layers} & 10 & 2 & 73.21 $\pm$ 3.05 & 66.56 $\pm$ 6.82 & 66.34 $\pm$ 6.75 & 83.77 $\pm$ 4.23 & 79.09 $\pm$ 4.73 & \textbf{84.58 $\pm$ 4.54} \\
   ~ & 10 & 3  & 70.28 $\pm$ 2.97 & 52.01 $\pm$ 6.52 & 51.79 $\pm$ 6.49 & 67.56 $\pm$ 5.64 & 60.71 $\pm$ 4.59 & \textbf{72.04 $\pm$ 4.75} \\
 \bottomrule
 \end{tabular}
\end{center}
\end{table*}

\begin{table*}[h]
\caption{Test accuracy of Convnet in CIFAR10}
 \label{table:cifar10}
\begin{center}
 \begin{tabular}{cccccccc}
  \toprule
  & S  & Local NN & FedAvg & FedProx & OT fusion & PFNM (FedMA) & NAFI\\
  \midrule
   \multirow{4}{*}{nets} & 5 & 25.21 $\pm$ 2.30 & 51.26 $\pm$ 2.97 & 51.43 $\pm$ 3.01 & 52.12 $\pm$ 2.78 & 50.39 $\pm$ 0.94 & \textbf{52.56 $\pm$ 0.81} \\
   ~ & 10 & 18.92 $\pm$ 1.41 & 46.78 $\pm$ 2.36 & 46.94 $\pm$ 2.32 & 45.43 $\pm$ 3.45 & 46.27 $\pm$ 1.73 & \textbf{47.46 $\pm$ 1.42} \\
   ~ & 15 & 15.85 $\pm$ 0.51 & 42.40 $\pm$ 2.25 & 42.06 $\pm$ 2.20 & 40.79 $\pm$ 4.29 & 42.82 $\pm$ 2.46 & \textbf{45.72 $\pm$ 1.35} \\
   ~ & 20 & 14.11 $\pm$ 0.53 & 34.20 $\pm$ 2.54 & 34.02 $\pm$ 2.52 & 37.56 $\pm$ 4.13 & 42.61 $\pm$ 2.07 & \textbf{44.96 $\pm$ 1.93} \\
 \bottomrule
 \end{tabular}
\end{center}
\end{table*}

\begin{table*}[htb]
\caption{Test dice coefficient of U-net in CIMC}
 \label{table:cimc}
\begin{center}
 \begin{tabular}{cccccccc}
  \toprule
  & S & Local NN & FedAvg & FedProx & OT fusion & PFNM (FedMA) & NAFI\\
  \midrule
   \multirow{2}{*}{nets} & 8 & 67.60 $\pm$ 9.38 & 53.28 $\pm$ 10.63 & 41.53 $\pm$ 0.82 & 63.57 $\pm$ 5.48 & 90.31 $\pm$ 2.90  & \textbf{96.47 $\pm$ 0.76} \\
   ~ & 16 & 27.84 $\pm$ 12.16 & 42.62 $\pm$ 0.87 & 48.50 $\pm$ 9.52 & 54.49 $\pm$ 5.62 & 75.47 $\pm$ 3.54  & \textbf{83.46 $\pm$ 3.41}
   \\
 \bottomrule
 \end{tabular}
\end{center}
\end{table*}

\subsubsection{Training setup}  We use PyTorch \citep{paszke2017automatic} to implement these networks and train them by the Adam  \citep{kingma2014adam}, SGD \citep{bottou2010large} and RMSprop \citep{hinton2012neural} with hyperparameter settings which are summarized in Supplement table~\ref{table:train_setting}.

\subsubsection{Performance Overview} 
 In real world applications of Federated Learning, the discrepancy of data distribution among local models and communication cost will inevitably increase as the number of local models increases. This is mainly due to the variability of data generation paradigms in the system \citep{li2020federated}. This suggests that testing fusing algorithms for various numbers of local models $S$ is important.  For MNIST, we firstly apply various baselines with 15, 20, 25 and 30 local models with one hidden layer FCNN. Local NN in Table~\ref{table:mnist} reports the average of separately tested network accuracies. The lower extremes of aggregating are determined by the performance of local NNs. As shown in Table~\ref{table:mnist}, accuracy of NAFI on MNIST is 4\% higher than PFNM on average. We also examine how the performance of all methods is impacted by the number of hidden layers $N$. We train 10 neural networks with 2 and 3 hidden layers respectively and then fuse them using various baselines. The performance of PFNM decreases with the increasing number of hidden layers—even worse than OT fusion. This phenomenon has already been proposed in \cite{wang2020federated}, and it can be fixed by a layer-wise training manner which we apply in CNN architectures. It is worth mentioning that NAFI seems have the potential to address the drawback of PFNM in deep architecture, as NAFI still maintain at a moderate level in deep architecture. This may be explained by an accumulating error effect. From above analysis, we discovered that PFNM is unable to discriminate those global neurons because the cost specification does not include global neuron information, while NAFI equipped with KL divergence can correct this. Thus, with the fusing process of neural networks is going layer-by-layer, the drawback of PFNM described above takes effect, and global neurons are not generated correctly. As a result, the incorrectness of PFNM in picking global neurons is superimposed, and the performance disparity between PFNM and NAFI is amplified as presented in Table~\ref{table:mnist}.

\par For CNN application, we apply ConvNet (2 convolutional layers and 3 fully connected layers) on 5, 10, 15, and 20 local models and use various methods to fuse the models trained in CIFAR 10's heterogeneous dataset partition. We apply PFNM and NAFI in an iteratively layer-wise way mentioned above, so we also set the communications rounds of FedAvg, FedProx and OT fusion  to 5 equal to the number of layers in ConvNet for equality. As shown in Table~\ref{table:cifar10}, NAFI outperforms the other methods on fusing convolutional neural networks. Concretely, the accuracy of NAFI on CIFAR 10 is 3\% higher than PFNM on average for a various number of local models.

\par U-Net \citep{ronneberger2015u} is widely used on medical image segmentation \citep{du2020medical}, and medical images such as tumor scan images generated by multiple institutions are generally forbidden to be exchanged for privacy issues. Therefore, it is critical to federate well-known image segmentation algorithms. To this end, we compare our method and baselines on U-Net. 
The U-Net architecture is described in \cite{ronneberger2015u}, and it employs the batch normalization layer, which is not extended by the matching type federated learning algorithm yet. Here we apply a technique which fuses batch normalization layer into convolutional layer, and extend the matching type algorithm to batch normalization layer. The details can be referred in Appendix E.
As shown in Table~\ref{table:cimc}, FedAvg, FedProx and OT fusion have very subpar performances. And FedAvg and FedProx even cannot converge to a stationary result in such complex network architecture (the results go up and down during communication rounds). The result shown in the table is taken from the best result during all 19 communication rounds which are equal to the number of layers in U-net.
As shown in Table~\ref{table:cimc}, the priority of NAFI is significantly more important among popular Federated Learning methods than it is in the case of 2- or 3-layer neural networks. With 8 local models, NAFI on CIMC has accuracy that is at least 7\% higher than PFNM, and accuracy that is almost 8\% higher with 16 local models.

\subsubsection{Amount of fused global model} In real federated application, the amount of fused global model determines the burden of communication. Hence we also measure the size of global model fused by PFNM and NAFI respectively. In figure~\ref{fig4}, we plot the logarithm ratio between the amount of fused global and the total amount of all local models $\log{\frac{J}{\sum_s{J_s}}}$ in FCNN multiple nets scenario. As shown in figure~\ref{fig4}, the logarithm ratio decrease with the increase in number of local models and is kept lower than -0.5, which means the matching type method produces a relative small model while can capture the most advantages of local models. It is worth noting that the ratio of NAFI is always small than PFNM, while the performance of fused global model in NAFI outperforms than PFNM accordint to above experiment results. This indicates the availability of NAFI in some model compression tasks.

\begin{figure}[htb]
\centering
\includegraphics[width=0.5\textwidth]{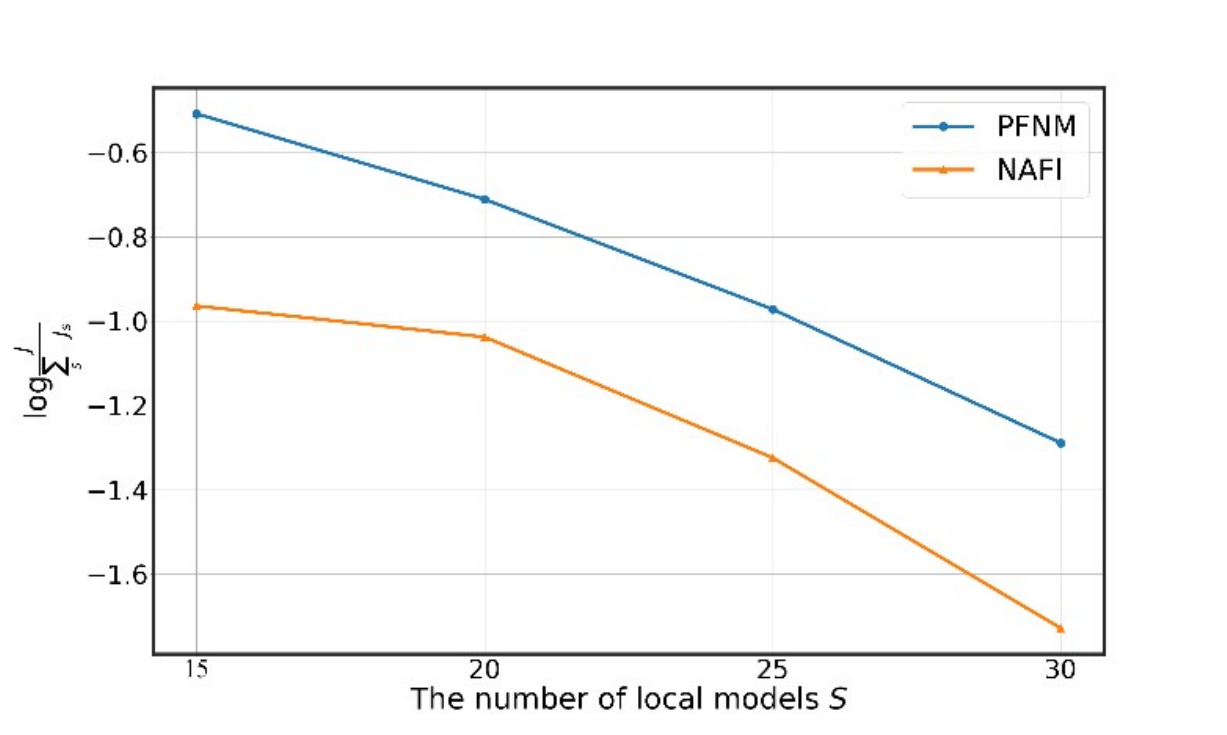}
\centering
\caption{\textbf{The logarithm ratio between the amount of fused global and the total amount of all local models in FCNN multiple nets scenario.}}
\label{fig4}
\end{figure}

\subsection{Sensitivity Analysis}
\par The weight of the KL-divergence term, $\lambda$, is a hyper-parameter that we introduce into NAFI. Therefore, it is essential to demonstrate the sensitivity of $\lambda$. Here, we set $\lambda$ to various positive values for NAFI applied on fusing FCNNs and CNNs trained in MNIST and CIFAR10 respectively. We also take the impact of the number of local models into account.  As shown in figure~\ref{fig3}, the heat map indicates the accuracy on the
test data, and for various number of clients, there is only a tiny fluctuation of prediction accuracy for a fused global model when $10^{-8} \leq \lambda \leq 1$. Although NAFI with $\lambda = 10^{-8}$ performs significantly worse than NAFI with $10^{-3} \leq \lambda \leq 0.5$,  the fused global model maintains a high level of performance.  When $\lambda = 1$, the performance of the fused global model drops sharply, this indicates that $\lambda$ should not set too high.  To summarize, NAFI is robust on the hyperparameter $\lambda$ under a variety of conditions. 
\begin{figure*}[htb]
\centering
\includegraphics[width=1\textwidth]{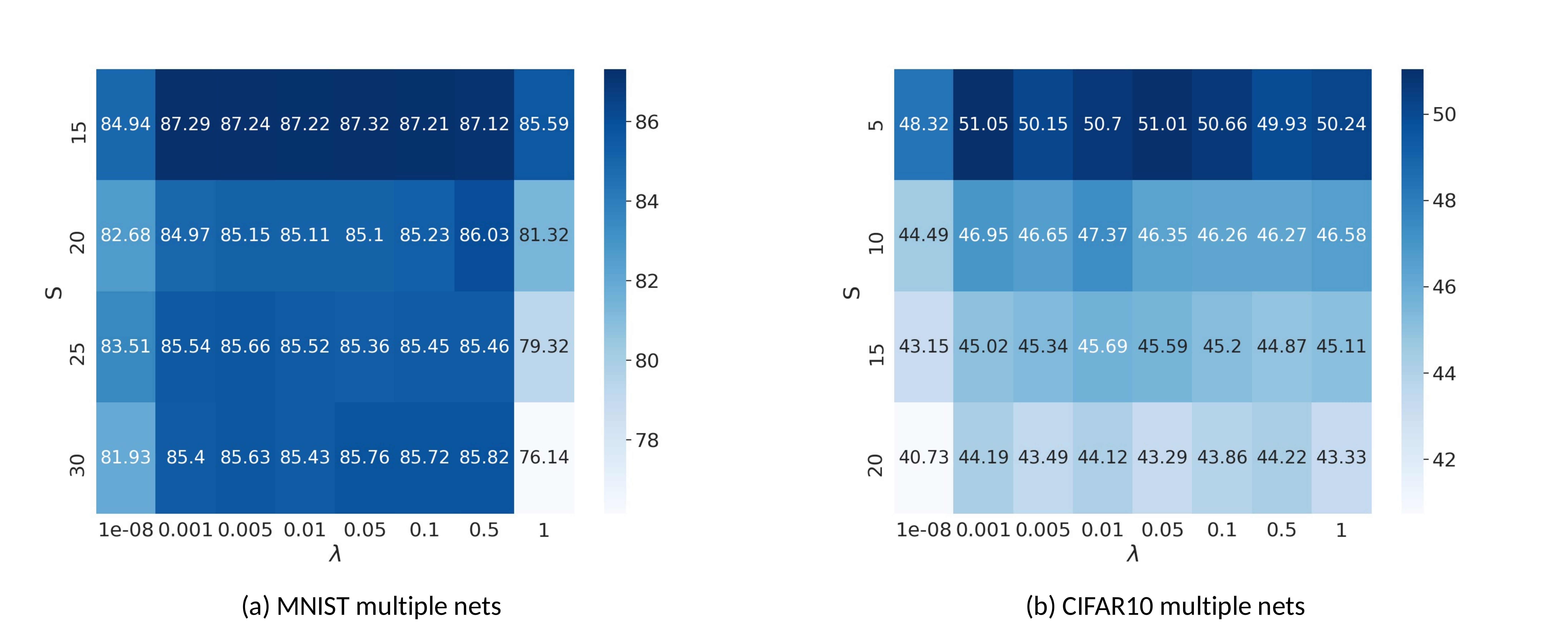}
\centering
\caption{\textbf{KL regularization coefficients sensitivity analysis}}
\label{fig3}
\end{figure*}

\section{Conclusion}
\label{conclusion}
In this study, we theoretically analyze the flaw of PFNM and propose a new federated neural matching method to fix this flaw. It is empirically shown that the new method outperforms other state-of-the-art algorithms for federated learning of neural networks, and the comparison of fused global model size also indicates the potential availability of our method in model compression and pruning tasks. 
In future work, we will extend our algorithm into more advanced architectures, like transformers etc. Additionally, the success of KL-divergence suggests it is also likely to try other metrics which can measure the similarity between two probability distributions.

\section*{Acknowledgements} 
During the process of preparing this paper, the FIFA World Cup Qatar 2022 is underwaying. The Argentina football player Lionel Messi and his Argentina men's football team have also successfully qualified for the group stage. For a long time, Messi has maintained a brilliant and lasting career through his own efforts and self-discipline. This kind of hard work and self-discipline also deeply moved the author, and Messi also become the spiritual mentor of the author during his career path. 
Lack of a World Cup trophy is the only barrier for Messi being the world's number one football player. In 2014, Messi regretted passing it by. Here, the author wishes Messi to win the World Cup trophy in this year through his effort. And all of us should believe that God rewards those who work hard. Let's work hard for our goals.

\clearpage
\bibliography{references}

\newpage

\onecolumn
\appendix

\section{Appendix A: BBP and IBP}
\label{Appendix2}
\textbf{Beta-Bernoulli Process and Indian Buffet Process}\quad Denote $Q$ as a random measure drawn from a Beta process: $Q | \gamma_0, H \sim \text{BP}(1, \gamma_0H)$, where $\gamma_0$ is the mass parameter, $H$ is the base measure over some domain $\Omega$ such that $H(\Omega) = 1$.
 One can show that $Q$ is a discrete measure with $Q=\sum_{i} q_i \delta_{\bmf{\theta}_{i}}$, which can be characterized by an infinitely countable set of (weight, atom) pairs $( q_i, \bmf{\theta}_i ) \in [0,1] \times \Omega$. The atoms $\bmf{\theta}_i$ can be drawn i.i.d from $H$ and the weights $\{q_i\}_{i=1}^{\infty}$ can be generated via a stick-breaking process \cite{teh2007stick}: $q_1 \sim \operatorname{Beta}(\gamma_0, 1), q_i = \prod_{g=1}^i q_g$. 
Then subsets of atoms in the random measure $Q$ can be picked via a Bernoulli process. That is, each subset $\mathcal{T}_s$ for $s = 1,\cdots, S$ can be distributed via a Bernoulli process with base measure $Q$: $\mathcal{T}_s | Q \sim \operatorname{BeP}(Q)$. Hence, subset $\mathcal{T}_s$ can also be viewed as a discrete measure $\mathcal{T}_s :=\sum_i a_{si}\delta_{\bmf{\theta}_i}$, which is formed by pairs $(a_{si}, \bmf{\theta}_i) \in \{0,1\} \times \Omega$, where $a_{si} | q_i \sim \operatorname{Bernoulli}(q_i), \forall i$ is a binary random variable indicating whether $\bmf{\theta}_i$ belongs to subset $\mathcal{T}_s$. We call such collection of subsets a Beta-Bernoulli process \cite{thibaux2007hierarchical}.

The Indian buffet process (IBP) specifies distribution on sparse binary matrices  \cite{ghahramani2006infinite}. IBP involves a metaphor of a sequence of customers tasting dishes in an infinite buffet: the first customer tastes $\operatorname{Poisson}(\gamma_0)$ dishes, every  subsequent  $s$th customer tastes each dish that is previously selected with probability $n_i / s$,  where $n_i = \sum_{s=1}^{S-1} a_{si}$, and then tastes  $\operatorname{Poisson}(\gamma_0 / s)$ new dishes. 
Marginalizing over Beta Process distributed $Q$ above will induce dependencies among subsets and recover the predictive distribution $\mathcal{T}_S | \mathcal{T}_1,\cdots,\mathcal{T}_{S-1} \sim \operatorname{BeP}(H\frac{\gamma_0}{S} + \sum_{i}\frac{n_i}{S}\delta_{\bmf{\theta}_i})$. That is equivalent to the IBP.

\section{Appendix B: Proof of Proposition 1 \cite{yurochkin2019bayesian}}
PFNM maximizes a posterior probability of the global atoms$\{\bmf{\theta}_i\}_{i=1}^\infty$ and assignments of observed neural network weight estimates to global atoms$\{ \bm{A}^s \}_{s=1}^S$.  Given estimates of the local model weights $\{ \bmf{w}_{sj} \text{ for } j = 1, \dots, J_s \}_{s=1}^S$, it has:
\begin{equation}
\label{eq16}
\max_{\{ \bmf{\theta}_i \}, \{ \bm{A}^s \} } P(\{ \bmf{\theta}_i\}, \{\bm{A}^s\} | \{\bmf{w}_{sj}\}) \\\propto P(\{\bmf{w}_{sj}\} | \{\bmf{\theta}_i\}, \{\bm{A}^s\} )P(\{\bm{A}^s\})P(\{\bmf{\theta}_i\}),
\end{equation}
by taking negative natural logarithm it can obtain:
\begin{equation}
\label{eq17}
\min_{\{ \bmf{\theta}_i \}, \{ \bm{A}^s \} } -\sum_i \bigg ( \sum_{s,j} \bmf{A}^s_{i,j} \log(p(\bmf{w}_{sj} | \sim \bmf{\theta}_i)) + \log(q(\bmf{\theta}_i)) \bigg ) \\
 - \log(P(\{ \bm{A}^s \})),
\end{equation}
expand probability function of multi-dimensional Gaussian distributions in equation~(\ref{eq17}), it obtains:
\begin{equation}
\label{eq18}
\min_{\{ \bmf{\theta}_i \}, \{ \bm{A}^s \} } \frac{1}{2} \sum_{i}\Bigg( \frac{|| \hat{\bmf{\theta}}_i - \bmf{\mu}_0 ||^2}{\sigma_0^2} + (D+K)\log(2\pi\sigma_0^2 ) \\
+ \sum_{s,j}A^s_{i,j} \frac{||\bmf{w}_{sj} - \hat{\bmf{\theta}}_i||^2}{\sigma_s^2}\Bigg)
 - \log(P(\{ \bm{A}^s \})).
\end{equation}
We now consider the first part of equation~(\ref{eq18}). Through the closed-form expression of $\{\bmf{\theta}_i\}$ estimated according to the Gaussian-Gaussian conjugacy:
\begin{equation}
\label{eq19}
\hat{\bmf{\theta}_i} =  \frac{\bmf{\mu}_0 / \sigma_0^2 + \sum_{s,j}A_{i,j}^s \bmf{w}_{sj} / \sigma^2_s}{1 / \sigma_0^2 + \sum_{s,j}A^s_{i,j} / \sigma^2_s} \text{ for } i = 1, ..., J,
\end{equation}
where for simplicity we assume $\bm{\Sigma}_0 = \bm{I} \sigma^2_0$ and $\bm{\Sigma}_s = \bm{I}\sigma^2_s$, we can now cast first part of equation~(\ref{eq18}) with respect only to $\{ \bm{A}^s \}_{s=1}^S$:
\begingroup
\allowdisplaybreaks
\begin{equation}
\label{eq20}
    \begin{aligned}
 		& \frac{1}{2} \sum_{i}\Bigg( \frac{|| \hat{\bmf{\theta}}_i - \bmf{\mu}_0 ||^2}{\sigma_0^2} + (D+K)\log(2\pi\sigma_0^2 ) + \sum_{s,j}A^s_{i,j} \frac{||\bmf{w}_{sj} - \hat{\bmf{\theta}}_i||^2}{\sigma_s^2} \Bigg) \\
 		 \cong &\frac{1}{2} \sum_{i}\Bigg( \langle \hat{\bmf{\theta}}_i, \hat{\bmf{\theta}}_i \rangle(\frac{1}{\sigma_0^2} + \sum_{s,j}\frac{A^s_{i,j}}{\sigma^2_s}) + (D+K)\log(2\pi\sigma_0^2 )  -  2\langle \hat{\bmf{\theta}}_i, \sum_{s,j}A^s_{i,j} \frac{\bmf{w}_{sj}}{\sigma_s^2}) \rangle\Bigg) \\
 		  = &-\frac{1}{2} \sum_{i}\Bigg(\frac{|| \sum\limits_{s,j} \bmf{A}^{s}_{i,j}\frac{\bmf{w}_{sj} - \bmf{\mu}_0}{\sigma_{s}^2}||^2}{( 1 / \sigma_0^2 + \sum\limits_{s,j} \bmf{A}^{s}_{i,j} / \sigma_{s}^2} - (D+K)\log(2\pi\sigma_0^2 )\Bigg). \\
    \end{aligned}
\end{equation}
\endgroup
Partition equation~(\ref{eq20}) between $i = 1,..., J_{-s'}$ and $i =J_{-s'} + 1,...,J_{-s'} + J_{s'}$, and because it is now solving for $ \bm{A}^{s'}$, it can subtract terms independent of $\bm{A}^{s'}$:
\begingroup
\allowdisplaybreaks
\begin{equation}
\label{eq21}
 	\begin{aligned}
 		& \sum_{i}\Bigg(\frac{|| \sum_{s,j} \bmf{A}^{s}_{i,j}\frac{\bmf{w}_{sj} - \bmf{\mu}_0}{\sigma_{s}^2}||^2}{( 1 / \sigma_0^2 + \sum_{s,j} \bmf{A}^{s}_{i,j} / \sigma_{s}^2} - (D+K)\log(2\pi\sigma_0^2 )\Bigg) \\
        \cong & \sum_{i=1}^{J_{-s'}} \Bigg(\frac{|| \sum_j \bmf{A}^{s'}_{i,j}\frac{\bmf{w}_{s'j} - \bmf{\mu}_0}{\sigma_{s'}^2} + \sum_{s \in -s',j} \bmf{A}^s_{i,j}\frac{\bmf{w}_{sj} - \bmf{\mu}_0}{\sigma_s^2} ||^2}{1 / \sigma_0^2 + \sum_{j} \bmf{A}^{s'}_{i,j} / \sigma_{s'}^2 + \sum_{s \in -s',j} \bmf{A}^s_{i,j} / \sigma^2_s} - \frac{||\sum_{s \in -s',j} \bmf{A}^s_{i,j}\frac{\bmf{w}_{sj} - \bmf{\mu}_0}{\sigma_s^2} ||^2}{1 / \sigma_0^2 + \sum_{s \in -s',j} \bmf{A}^s_{i,j} / \sigma^2_s } \Bigg) \\
        & \quad  + \sum_{i=J_{-s'}+1}^{J_{-s'} + J_{s'}} \Bigg(\frac{|| \sum_j \bmf{A}^{s'}_{i,j}\frac{\bmf{w}_{s'j} - \bmf{\mu}_0}{\sigma_{s'}^2} ||^2}{1 / \sigma_0^2 + \sum_j \bmf{A}^{s'}_{i,j} / \sigma^2_{s'}}\Bigg),
 	\end{aligned}
\end{equation}
\endgroup
observe that $\sum_j \bmf{A}^{s'}_{i,j} \in \{ 0,1 \}$, i.e. it is 1 if some neuron from dataset $s'$ is matched to global neuron $i$ and 0 otherwise. Thus equation~(\ref{eq21}) can rewritten as a linear sum assignment problem:
\begingroup
\allowdisplaybreaks
\begin{equation}
\label{eq22}
	\begin{aligned}
	  &\sum_{i=1}^{J_{-s'}} \sum_{j=1}^{J_{s'}} \bmf{A}^{s'}_{i,j} \Bigg(\frac{||\frac{\bmf{w}_{s'j} - \bmf{\mu}_0}{\sigma_{s'}^2} + \sum_{s \in -s',j} \bmf{A}^s_{i,j}\frac{\bmf{w}_{sj} - \bmf{\mu}_0}{\sigma_s^2} ||^2}{1 / \sigma_0^2 + 1 / \sigma_{s'}^2 + \sum_{s \in -s',j} \bmf{A}^s_{i,j} / \sigma^2_s} - \frac{||\sum_{s \in -s',j} \bmf{A}^s_{i,j}\frac{\bmf{w}_{sj} - \bmf{\mu}_0}{\sigma_s^2} ||^2}{1 / \sigma_0^2 + \sum_{s \in -s',j} \bmf{A}^s_{i,j} / \sigma^2_s}\Bigg) \\
	   & \quad +  \sum_{i=J_{-s'}+1}^{J_{-s'} + J_{s'}} \sum_{j=1}^{J_{s'}} \bmf{A}^{s'}_{i,j} \Bigg( \frac{||\frac{\bmf{w}_{s'j} - \bmf{\mu}_0}{\sigma_{s'}^2} ||^2}{1 / \sigma_0^2 + \sum_j \bmf{A}^{s'}_{i,j} / \sigma^2_{s'}}\Bigg).
	\end{aligned}
\end{equation}
\endgroup
Then consider the second term of equation~(\ref{eq18}), by subtracting terms independent of $\bm{A}^{s'}$ it has:
\begin{equation}
\label{eq23}
	\log(P(\bm{A}^{s'})) = \log(P(\bm{A}^{s'} | \bm{A}^{-s'})) + \log(P(\bm{A}^{-s'})).
\end{equation}
First, it can ignore $\log(P(\bm{A}^{-s'}))$ since now are optimizing for $\bm{A}^{s'}$ . Second, due to exchange ability of datasets
(i.e. customers of the IBP), $\bm{A}^{s'}$ can always be treated
as the last customer of the IBP. Denote
$n^{-s'}_i = \sum_{-s',j} \bmf{A}^{s'}_{i,j}$ as the number of times local weights were assigned to global atom $i$ outside of group $s'$. Now it can obtain the following:
\begingroup
\allowdisplaybreaks
\begin{equation}
\label{eq24}
\begin{aligned}
  \log P(\bm{A}^{s'} | \bm{A}^{-s'}) & \cong \sum_{i=1}^{J_{-s'}} \Bigg(\bigg(\sum_{j=1}^{J_{s'}} \bmf{A}^{s'}_{i,j}\bigg)\log\frac{n^{-s'}_i}{S} + \bigg(1 - \sum_{j=1}^{J_{s'}} \bmf{A}^{s'}_{i,j}\bigg)\log\frac{S - n^{-s'}_i}{S}\Bigg)  \\
	& \qquad -  \log\Bigg(\sum_{i=J_{-s'}+1}^{J_{-s'} + J_{s'}} \sum_{j=1}^{J_{s'}} \bmf{A}^{s'}_{i,j}\Bigg) + \Bigg(\sum_{i=J_{-s'}+1}^{J_{-s'} + J_{s'}} \sum_{j=1}^{J_{s'}} \bmf{A}^{s'}_{i,j}\Bigg)\log\frac{\gamma_0}{J}.
\end{aligned}
\end{equation}
\endgroup
equation~(\ref{eq24}) thus can be rearranged as a linear sum assignment problem:
\label{eq25}
\begin{equation}
   \sum_{i=1}^{J_{-s'}} \sum_{j=1}^{J_{s'}} \bmf{A}^{s'}_{i,j}\log\frac{ n^{-s'}_i}{S - n^{-s'}_i}  +  \sum_{i=J_{-s'}+1}^{J_{-s'} + J_{s'}} \sum_{j=1}^{J_{s'}} \bmf{A}^{s'}_{i,j}\Bigg(\log\frac{\gamma_0}{S} - \log(i - J_{-s'}) \Bigg).
\end{equation}
Combining equation~(\ref{eq22}) and equation~(\ref{eq25}), we arrive at the cost specification shown in (6) of the main text. That completes the proof of Proposition 1 in the main text.

\section{Appendix C: Proof of Proposition 2}
\begin{proof}
i) Let's prove the first part. From Bayesian theory, the posterior is given by
  \begin{align}
      p(\bmf{\theta}_i \vert \bmf{w}_{sj}) & \propto p(\bmf{\theta}_i)p(\bmf{w}_{sj}) \\
      & \propto \exp{(\frac{-\bmf{\theta}_i^T\bmf{\theta}_i+2\bmf{\theta}_i^T\bmf{\mu}_0-\bmf{\mu}_0^T\bmf{\mu}_0}{2\sigma_0^2})} \\
      &  \exp{(\frac{-\bmf{w}_{sj}^T\bmf{w}_{sj}+2\bmf{\theta}_i^T\bmf{w}_{sj}-\bmf{\theta}_i^T\bmf{\theta}_i}{2\sigma_s^2})}. \\
  \end{align}
  From Gaussian conjugate, the product of two Gaussians is still a Gaussian, we will rewrite this in the form
  \begin{align}
      p(\bmf{\theta}_i \vert \bmf{w}_{sj}) \propto 
      &  \exp{\bigl[-\frac{\bmf{\theta}_i^T\bmf{\theta}_i}{2}(\frac{1}{\sigma_0^2} + \frac{1}{\sigma_s^2})+\bmf{\theta}_i^T(\frac{\bmf{\mu}_0}{\sigma_0^2}+\frac{\bmf{w}_{sj}}{\sigma_s^2}) - (\frac{\bmf{\mu}_0^T\bmf{\mu}_0}{2\sigma_0^2}+\frac{\bmf{w}_{sj}^T\bmf{w}_{sj}}{2\sigma_s^2})\bigr]} \\
      & \mathop{=}\limits^{def} \exp{[\frac{-\bmf{\theta}_i^T\bmf{\theta}_i+2\bmf{\theta}_i^T\bmf{\hat{\mu}}-\bmf{\hat{\mu}}^T\bmf{\hat{\mu}}}{2\hat{\sigma}^2}]}
      = \exp{[-\frac{1}{2\hat{\sigma}^2}(\norm{\bmf{\theta}_i - \bmf{\hat{\mu}}}^2)]},
  \end{align}
  where $\bmf{\hat{\mu}}$ and $\hat{\sigma}^2$ denote the mean and variance of the posterior Gaussian. By completing the square, we first match the coefficients of $\bmf{\theta}_i^T\bmf{\theta}_i$---second power of $\bmf{\theta}_i$, and find $\hat{\sigma}^2$ is given by
  \begin{align}
      \frac{1}{\hat{\sigma}^2} & = \frac{1}{\sigma_0^2} + \frac{1}{\sigma_s^2},\\
      \hat{\sigma}^2 &= \frac{1}{\frac{1}{\sigma_0^2} + \frac{1}{\sigma^2_{s}}},
  \end{align}
  and then match the coefficients of first power $\bmf{\theta}_i^T$ we get
  \begin{align}
      \frac{\bmf{\hat{\mu}}}{\hat{\sigma}^2} &= \frac{\bmf{\mu}_0}{\sigma_0^2}+\frac{\bmf{w}_{sj}}{\sigma_s^2},
  \end{align}
  hence
  \begin{align}
      \bmf{\hat{\mu}} = \hat{\sigma}^2(\frac{\bmf{\mu}_0}{\sigma_0^2}+\frac{\bmf{w}_{sj}}{\sigma_s^2}).
  \end{align}
  ii) Now let's prove the second part. When $\bmf{\theta}_i$ has been assigned local neurons from other assignments $\bmf{Z}_i^{-s'} = \{(s, j) \vert A_{ij}^s = 1, s \in -s'\}$,from Bayesian theory, the posterior is given by
  \begin{align}
      p(\bmf{\theta}_i \vert \bmf{Z}_i^{-s'}) & \propto p(\bmf{\theta}_i)\prod_{z\in \bmf{Z}_i^{-s'}}p(\bmf{w}_{z}) \\
      & \propto \exp{(\frac{-\bmf{\theta}_i^T\bmf{\theta}_i+2\bmf{\theta}_i^T\bmf{\mu}_0-\bmf{\mu}_0^T\bmf{\mu}_0}{2\sigma_0^2})} \\
      &  \exp{(-\sum_{s\in {-s'},j} A_{i,j}^s\frac{\bmf{w}_{sj}^T\bmf{w}_{sj}}{\sigma_s^2}  +  \bmf{\theta}_i^T\sum_{s\in {-s'},j} A_{i,j}^s\frac{\bmf{w}_{sj}}{\sigma_s^2}  -  \bmf{\theta}_i^T\bmf{\theta}_i \sum_{s\in {-s'},j} A_{i,j}^s\frac{1}{2\sigma_s^2}
      )}.
  \end{align}
  Since the product of two Gaussians is a Gaussian, we will rewrite this in the form
  \begin{align}
      p(\bmf{\theta}_i \vert \bmf{Z}_i^{-s'}) & \propto \exp{\bigl[-\frac{\bmf{\theta}_i^T\bmf{\theta}_i}{2}(\frac{1}{\sigma_0^2} + \sum_{s\in {-s'},j} A_{i,j}^s\frac{1}{\sigma_s^2}) +\bmf{\theta}_i^T(\frac{\bmf{\mu}_0}{\sigma_0^2} + \sum_{s\in {-s'},j} A_{i,j}^s\frac{\bmf{w}_{sj}}{\sigma_s^2}) - (\frac{\bmf{\mu}_0^T\bmf{\mu}_0}{2\sigma_0^2} + \sum_{s\in {-s'},j} A_{i,j}^s\frac{\bmf{w}_{sj}^T\bmf{w}_{sj}}{2\sigma_s^2})\bigr]} \\
      & \mathop{=}\limits^{def} \exp{(\frac{-\bmf{\theta}_i^T\bmf{\theta}_i+2\bmf{\theta}_i^T\bmf{\hat{\mu}}-\bmf{\hat{\mu}}^T\bmf{\hat{\mu}}}{2\hat{\sigma}^2})}
      = \exp{[-\frac{1}{2\hat{\sigma}^2}(\norm{\bmf{\theta}_i - \bmf{\hat{\mu}}}^2)]},
  \end{align}
  where $\bmf{\hat{\mu}}$ and $\hat{\sigma}^2$ denote the mean and variance of the posterior Gaussian. By completing the square, we first match the coefficients of $\bmf{\theta}_i^T\bmf{\theta}_i$---second power of $\bmf{\theta}_i$, and find $\hat{\sigma}^2$ is given by
  \begin{align}
      \frac{1}{\hat{\sigma}^2} & = \frac{1}{\sigma_0^2} + \sum_{s\in {-s'},j} A_{i,j}^s\frac{1}{\sigma_s^2},\\
      \hat{\sigma}^2 &= \frac{1}{\frac{1}{\sigma_0^2} + \sum_{s\in {-s'},j} A_{i,j}^s\frac{1}{\sigma_s^2}},
  \end{align}
  and then match the coefficients of first power $\bmf{\theta}_i^T$ we get
  \begin{align}
      \frac{\bmf{\hat{\mu}}}{\hat{\sigma}^2} &= \frac{\bmf{\mu}_0}{\sigma_0^2} + \sum_{s\in {-s'},j} A_{i,j}^s\frac{\bmf{w}_{sj}}{\sigma_s^2},
  \end{align}
  hence
  \begin{align}
      \bmf{\hat{\mu}} = \hat{\sigma}^2(\frac{\bmf{\mu}_0}{\sigma_0^2} + \sum_{s\in {-s'},j} A_{i,j}^s\frac{\bmf{w}_{sj}}{\sigma_s^2}).
  \end{align}
\end{proof}

\section{Appendix D: Iteratively Layer-wise Matched Aggregation via NAFI}
As the empirical study in \cite{wang2020federated} demonstrates, directly applying the matching algorithms  fails on deep architectures which are necessary to solve more complex tasks. Thus to alleviate this problem ,
we also extend NAFI to the following layer-wise matching scheme. Firstly, the server only collects the first layer weights from the clients and applies NAFI
to acquire the first layer weights of the federated model. Then, the server broadcasts these weights
to the clients, which proceed to train all consecutive layers on their datasets while keeping the matched layers frozen. Repeat this process until the last layer, where we make a weighted average based on class proportions of each client's data points.  We summarize our
 layer-wise version of NAFI in Algorithm 2. The layer-wise approach requires communication
rounds that equal to the number of layers in a neural network. Experimental results show that with layer-wise
matching, NAFI performs well on the ConvNets even for U-nets which has a complex architecture. In the more challenging
heterogeneous setting, NAFI outperforms FedAvg, FedProx trained with same number of
communication rounds (5 for ConvNet and 19 for U-net).
\begin{algorithm}
\caption{Iteratively Layer-wise NAFI}
\label{alg2}
\begin{algorithmic}[1]
    \REQUIRE ~~\\
    Collected local weights of $N$-layer architectures $\{ \bmf{W}_s^{(0)}, \cdots, \bmf{W}_s^{(N-1)} \}_{s=1}^{S}$ from $S$ clients;
    \ENSURE ~~\\
    New constructed global weights $\{ \bmf{W}^{(0)}, \cdots, \bmf{W}^{(N-1)} \}$.
 	\STATE $n = 0$;
 	\WHILE{layers $ n \leq N$}
 	    \IF{n < N-1}
 		    \STATE $\{A_s \}_{s=1}^S = \operatorname{NAFI}(\{ \bmf{W}_s^{(n)} \}_{s=1}^S)$;
 		    \STATE $\bmf{W}^{(n)} = \frac{1}{S}\sum_s \bmf{W}_s^{(n)}A^T_s$;
 		\ELSE 
 		    \STATE $\bmf{W}^{(n)} = \sum_s \bmf{p_{s}} \bmf{\cdot} \bmf{W}^{(n)}_s $ where $\bmf{p_{s}}$ is vector of fraction of data points with each label on worker $s$, and $\bmf{\cdot}$ denotes the dot product;
 		  \ENDIF
 		  \FOR{ $s \in \{ 1, \cdots,S \}$ }
 		    \STATE $\bmf{W}_s^{(n+1)} = A_s \bmf{W}_s^{(n+1)}$;
 		    \STATE Train $\{ \bmf{W}_s^{(n+1)},\cdots, \bmf{W}_s^{(N-1)}\}$ with $\bmf{W}_s^{(n)}$ frozen;
 		  \ENDFOR
 	    \STATE $n = n+1$
 	\ENDWHILE
\end{algorithmic}
\end{algorithm}

\section{Appendix E: Extending federated matching algorithm to Batch Normalization Layer}
\label{appendix4}
Although \cite{wang2020federated} shows how to apply the PFNM to CNNs, it doesn't enable additional deep learning building blocks, e.g., batch normalization layer, to the matching algorithm. However, widely used deep CNNs such as U-net often contain batch normalization layer in their architectures. In this paper, we utilize a common setup which merges the batch normalization layer with a preceding convolution to incorporate GI-FNM with batch normalization layer.

Without loss of generality, we assume that the feature map size is the same as the filter size.  Let $W_{\rm{conv}} \in \mathbb{R}^{C^{out} \times (C^{in}\cdot w \cdot h)}$ and $b_{\rm{conv}} \in \mathbb{R}^{C^{out}}$ be the parameters of the convolutional layer that precedes batch normalization. Given a input data $x \in \mathbb{R}^{(C^{in}\cdot w \cdot h)}$, the convolutional operator can be simply expressed as:
\begin{equation}\label{conv_layer}
    f = W_{\rm{conv}} * x + b_{\rm{conv}}.
\end{equation}
Batch normalization (BN) is a popular method used in modern neural networks as it often reduces training time and potentially improves generalization.  Given the outputting feature of preceding convolutional layer, it can be normalized as follows:
\begin{equation}\label{batch_normlization}
    \hat{f} = \gamma\frac{f-\mu}{\sqrt{\sigma^2 + \epsilon}} + \beta,
\end{equation}
where $\mu$ and $\sigma^2$ are the mean and variance computed over a batch of feature, $\epsilon$ is a small constant included for numerical stability, $\gamma$ is the scaling factor and $\beta$ the shift factor. The parameters $\gamma$ and $\beta$ are slowly learned with gradient descent together with the other parameters of the network. 

If we take $W_{\rm{conv}}$ and $b_{\rm{conv}}$ into the Eq.~(\ref{batch_normlization})
we can get the new weights and bias as:
\begin{itemize}
    \item weights: $W_{\rm{BN}} = \gamma\cdot\frac{W_{\rm{conv}}}{\sqrt{\sigma^2 + \epsilon}}$;
    \item bias: $b_{\rm{BN}} = \gamma\cdot\frac{ (b_{\rm{conv}} - \mu)}{\sqrt{\sigma^2 + \epsilon}} + \beta$.
\end{itemize}
Thus the batch normalized feature can be directly obtained by:
\begin{equation}\label{bn_conv_layer}
    \hat{f} = W_{\rm{BN}} * x + b_{\rm{BN}}.
\end{equation}
By matching the fused weights $W_{\rm{BN}}$ and bias $b_{\rm{BN}}$, we enable the batch normalization layer in GI-FNM.


\end{document}